\documentclass[12pt,draftcls,onecolumn]{article} 
\usepackage{amssymb,amsmath,amsfonts,epsfig,graphicx,mathrsfs,amsthm,array}

\def\BibTeX{{\rm B\kern-.05em{\sc i\kern-.025em b}\kern-.08em
    T\kern-.1667em\lower.7ex\hbox{E}\kern-.125emX}}

\newtheorem{theorem}{Theorem}
\newtheorem{lemma}{Lemma}

\newcommand{\ocap}{\textcircled{$\scriptstyle{\cap}$}}

\newcommand{\ocup}{\textcircled{$\scriptstyle{\cup}$}}

\begin{document}

\title{A geometric approach to conditioning belief functions}

\author{Fabio Cuzzolin
\thanks{F. Cuzzolin is with the School of Engineering, Computing and Mathematics, Oxford Brookes University, Oxford, UK. Email: Fabio.Cuzzolin@brookes.ac.uk.}
}


\maketitle

\begin{abstract}
Conditioning is crucial in applied science when inference involving time series is involved. Belief calculus is an effective way of handling such inference in the presence of epistemic uncertainty -- unfortunately, different approaches to conditioning in the belief function framework have been proposed in the past, leaving the matter somewhat unsettled. 
Inspired by the geometric approach to uncertainty, in this paper we propose an approach to the conditioning of belief functions based on geometrically projecting them onto the simplex associated with the conditioning event in the space of all belief functions. We show here that such a geometric approach to conditioning often produces simple results with straightforward interpretations in terms of degrees of belief. This raises the question of whether classical approaches, such as for instance Dempster's conditioning, can also be reduced to some form of distance minimisation in a suitable space. The study of families of combination rules generated by (geometric) conditioning rules appears to be the natural prosecution of the presented research.
\end{abstract}

\section{Introduction} \label{sec:introduction}

Decision making and estimation are common problems in applied science, as people or machines need to make inferences about the state of the external world and take appropriate actions. Such state is typically assumed to be described by a probability distribution over a set of alternative hypotheses, which in turn needs to be inferred from the available data. Sometimes, however, as in the case of extremely rare events (e.g., a volcanic eruption), few samples are available to drive such inference. Part of the data can be missing. In addition, under the law of large numbers, probability distributions are the outcome of an infinite process of evidence accumulation whereas, in all practical cases, the available evidence can only provide some sort of constraint on the unknown, `true' probability governing the process.

Different kinds of constraints are associated with different generalisations of probabilities, formulated to model `epistemic' uncertainty at the level of probability distributions \cite{denoeux2020representations}. The simplest such generalisations are, arguably, interval probabilities \cite{decampos94} and convex sets of probabilities or `credal sets' \cite{levi80book}. A whole battery of different uncertainty theories \cite{walley91book} has been developed in the last century or so, starting from De Finetti's pioneering work \cite{definetti74}. In particular, Shafer's theory of belief functions (b.f.s) \cite{Shafer76}, based on A. Dempster's \cite{dempster67multivalued} seminal work, allows us to express partial belief by providing lower and upper bounds to probability values \cite{yager2008classic}. The widespread presence and influence of uncertainty at different levels explains why belief functions have been increasingly applied to fields as diverse as robotics \cite{reineking2016active}, economics, computer vision \cite{bloch2008defining,lian2018joint}, communications \cite{alshamaa2017tracking,senouci2014belief} or cybersecurity \cite{sun2006information}. Powerful tools for decision making \cite{strat94decisionanalysis,smets05ijar,troffaes07} and classification \cite{liu2019evidence} with belief functions have also been proposed -- see \cite{denoeux2019decision} for a recent review.

When observations come from time series, however, or when conditional independence assumptions are necessary to simplify the structure of the joint (belief) distribution to estimate, the need for generalising the classical results on total probability to belief calculus arises.
This is the case, for instance, in computer vision problems such as image segmentation \cite{minaee2021image}, where conditional independence is crucial to make large-scale optimisation problems tractable. In target tracking, conditional constraints on the targets' future positions given their past locations are available. If such constraints are described as belief functions \cite{cuzzolin00mtns,gennari02-integrating,cuzzolin05isipta,cuzzolin13fusion,gong2017belief}, predicting current target locations requires combining conditional belief functions into a single `total' belief function \cite{cuzzolin01thesis,Moral96,zhou2017total,dezert2018total,dezert2018total-fusion}, in a generalisation of the total probability theorem.


Unfortunately, in opposition to what happens in classical probability theory, different definitions of conditional belief functions \cite{kulasekere2004conditioning} can be imagined and have indeed been proposed in the past \cite{fagin91new,smets93jeffrey,suppes1977,Klawonn:1992:DBT:2074540.2074558,spies94conditional,klopotek99lncs,tang05dempster,lehrer05updating}. 
Quite recently, the idea of formulating the problem geometrically has emerged. Lehrer \cite{lehrer05updating}, in particular, has proposed such a geometric approach to determine the conditional expectation of non-additive probabilities (such as belief functions). 

The notion of generating conditional belief functions by minimising a suitable distance function between the original b.f. $b$ and the `conditioning region' $\mathcal{B}_A$ associated with the conditioning event $A$, i.e., the set of belief functions whose b.b.a. assigns mass to subsets of $A$ only
\begin{equation} \label{eq:geometric-conditioning}
b_d(.|A) = \arg \min_{b' \in \mathcal{B}_A} d(b,b'),
\end{equation}
naturally arises in the context of the geometric approach to uncertainty \cite{cuzzolin2021springer}. This original angle has a clear potential, as it expands our arsenal of approaches to the problem and, we argue here, is a promising candidate to the role of general framework for conditioning. 

\subsection{Contributions and paper outline}

In particular, in this paper we explore the geometric conditioning problem as defined above in the mass space representation 
in which belief functions are represented by the vectors of their mass values, 
and adopt distance measures of the classical Minkowski ($L_p$) family. 

We first recall in Section \ref{sec:related-work} the state of the art concerning evidence theory (Section \ref{sec:soa-evidence}), the conditioning of belief functions (Section \ref{sec:soa-conditioning}), the geometric approach to uncertainty (Section \ref{sec:soa-geometric-approach}) and the use of distances in evidence theory (Section \ref{sec:soa-distances}).

The core of the paper is Section \ref{sec:conditional-m}, in which our main results on conditioning belief functions in the mass space are derived.
We first recall in Section \ref{sec:geometric} the basic notions of the geometric approach to belief functions. In particular, we show how each b.f. can be represented by either the vector of its belief values in the belief space or the vector of its mass values in the mass space. 
The notion of geometric conditional belief function and the associated minimisation problem are formalised in Section \ref{sec:notion-conditional}.
In Sections \ref{sec:l1-m}, \ref{sec:l2-m} and \ref{sec:linf-m}, instead, we prove the analytical forms of the $L_1$, $L_2$ and $L_\infty$ conditional belief functions in the mass space, respectively. 

The results obtained are discussed in Section \ref{sec:discussion}, by summarising the main results (Section \ref{sec:discussion-summary}), highlighting the properties of geometric conditional belief functions (Section \ref{sec:discussion-properties}), hinting at an interesting connection with Lewis' imaging \cite{lewis76} generalized to b.f.s (Section \ref{sec:interpretation-imaging}) and comparing the results of geometrically conditioning belief functions in the mass as opposed to the belief space  (Section \ref{sec:discussion-comparison}).

Finally, in Section \ref{sec:perspectives} we prospect a number of future developments for the geometric approach to conditioning. 

\section{State of the art} \label{sec:related-work}

\subsection{Belief functions} \label{sec:soa-evidence}

A \emph{basic probability assignment} (b.p.a.) over a finite set \emph{(frame of discernment \cite{Shafer76})} $\Theta$ is a function $m_b : 2^\Theta\rightarrow [0,1]$ on its power set $2^\Theta = \{A \subseteq \Theta\}$ such that 
\[
m_b(\emptyset) = 0,  \quad
\sum_{A \subseteq \Theta} m_b(A) = 1. 
\]
Subsets of $\Theta$ associated with non-zero values of $m_b$ are called
\emph{focal elements}. 

The \emph{belief function} $b:2^\Theta\rightarrow [0,1]$ associated with a basic probability assignment $m_b$ on $\Theta$ is defined as
\begin{equation} \label{eq:belief-function}
b(A) = \sum_{B \subseteq A} m_b(B). 
\end{equation}
The \emph{core} $\mathcal{C}_b$ of a b.f. $b$ is the union of its focal elements. 

A dual mathematical representation of the evidence encoded by a belief function $b$ is the \emph{plausibility function} $pl_b : 2^{\Theta}  \rightarrow [0,1]$, $A \mapsto pl_b(A)$, where the plausibility value $pl_b(A)$ of an event $A$ is given by 
\[
pl_b(A) \doteq 1 - b({A}^c) = \sum_{B \cap A \neq \emptyset} m_b(B) 
\]
and expresses the amount of evidence \emph{not against} $A$.

The \emph{orthogonal sum} $b_1 \oplus b_2 : 2^\Theta \rightarrow [0,1]$ of two belief functions $b_1 : 2^\Theta \rightarrow [0,1]$, $b_2 : 2^\Theta \rightarrow [0,1]$ defined on the same frame $\Theta$ is the unique belief function on $\Theta$ whose focal elements are all the possible intersections $A = B \cap C$ of focal elements $B$ and $C$ of $b_1$ and $b_2$, respectively, and whose basic probability assignment is given by
\[
\displaystyle m_{b_1 \oplus b_2}(A) = \frac{ \displaystyle \sum_{B \cap C = A} m_{b_1}(B) m_{b_2} (C)} { \displaystyle 1 - \sum_{B \cap C = \emptyset} m_{b_1} (B) m_{b_2} (C)}.
\]

\subsection{Conditioning in belief calculus} \label{sec:soa-conditioning}

\subsubsection{Dempster conditioning}

The original proposal is due to Dempster himself \cite{dempster67multivalued}. He formulated it in his original model, in which belief functions are induced by multi-valued mappings $\Gamma : \Omega \rightarrow 2^\Theta$ of probability distributions defined on a set $\Omega$ onto the power set of another set (`frame') $\Theta$. 

Given a conditioning event $B \subset \Theta$, the `logical' (or `categorical', in Smets's terminology) belief function $b_B$ such that $m_b(B) = 1$ is combined via Dempster's rule with the a priori BF $b$. The resulting measure $b \oplus b_B$ is the {conditional belief function given $B$} \emph{\`a la Dempster}, which we denote by $b_\oplus(A|B)$. 

Dempster conditioning was almost immediately and strongly criticised from a Bayesian standpoint.
In response to these objections a number of approaches to conditioning in belief calculus have been proposed  \cite{kyburg87bayesian,Chateauneuf89,fagin91new,Jaffray92,GILBOA199333,Denneberg1994,yu94conditional,itoh95new,tang05dempster} along the years, in different mathematical setups. 

\subsubsection{Credal conditioning}

Fagin and Halpern, in particular, argued that Dempster conditioning behaves unreasonably in the context of the classical `three prisoners' example \cite{fagin91new}, originally discussed by Diaconis \cite{diaconis78review}.
They thus proposed a notion of conditional belief \cite{fagin91new} as the lower envelope of a family of conditional probability functions, and provided a closed-form expression for it. 
The notion, quite related to the concept of `inner measure', was considered by various other authors \cite{kyburg87bayesian}.

Namely, a b.f. $b$ can be interpreted as the lower envelope of a family $\mathcal{P}[b]$ of probability distributions, i.e.
\[
b (A) = \inf_{P \in \mathcal{P}[b]} P(A).
\]
Fagin and Halpern thus defined the conditional belief function associated with $b$ as the {lower envelope} (that is, the infimum) {of the family of conditional probability functions $P(A|B)$}, where $P$ is consistent with $b$:
\begin{equation} \label{eq:conditioning-fagin}
b_{\text{Cr}}(A|B) \doteq \inf_{P \in \mathcal{P}[b]} P(A|B), 
\quad 
pl_{\text{Cr}}(A|B) \doteq \sup_{P \in \mathcal{P}[b]} P(A|B).
\end{equation}

This `credal' conditioning is more conservative than Dempster conditioning, as the associated probability interval is included in that resulting from Dempster conditioning.

\subsubsection{Spies's definition}

In the (original) framework of multi-valued mappings, instead, Spies \cite{spies94conditional} defined conditional events as \emph{sets of equivalent events} under conditioning. By applying multi-valued mapping to such events, conditional belief functions can be introduced. An updating rule, which is equivalent to the law of total probability if all beliefs are probabilities, was also introduced. 

Importantly, in this definition a conditional belief function is \emph{not} a b.f. on the subalgebra of the subsets of the conditioning event $A$.
It can be proven that Spies's conditional belief functions are closed under Dempster's rule of combination and therefore coherent with the random-set interpretation of the theory.

In related work by Slobodova \cite{slobodova94conditional}, a multi-valued extension of conditional belief functions was introduced \cite{Slobodova97} and its properties examined. 

\subsubsection{Geometric conditioning}

Under the \emph{focusing} principle of belief revision \cite{dubois1992evidence,dubois1997focusing,dubois1993belief}, no new information is introduced as we merely focus on a specific subset of the original hypothesis space. When applied to belief functions, this yields Suppes and Zanotti's \emph{geometric conditioning} \cite{suppes1977}:
\begin{equation} \label{eq:geometric-conditioning-values}
b_\text{G}(A|B) = \frac{b(A \cap B)}{b(B)}, 
\quad 
pl_\text{G}(A|B) = \frac{b(B) - b(B \setminus A)}{b(B)}.
\end{equation}
This was proved by Smets using Pearl's `probability of provability' \cite{Pearl:1988:PRI:52121} interpretation of belief functions. 

It is interesting to note that geometric conditioning is somewhat dual to Dempster conditioning, as it amounts to {replacing probability values with belief values} in Bayes' rule:
\[
\begin{array}{ccc}
pl_\oplus(A|B) = \frac{\displaystyle pl_b (A \cap B)}{\displaystyle pl_b (B) } & \leftrightarrow & b_\text{G}(A|B) = \frac{\displaystyle b(A \cap B)}{\displaystyle b(B)}.
\end{array}
\]
Unlike Dempster and conjunctive conditioning (see the next paragraph), geometric conditioning does not seem to be linked to an underlying combination operator, although the issue remains open.

\subsubsection{Smets's conjunctive rule of conditioning} 

Another way of dealing with the classical Bayesian criticism of Dempster's rule is to abandon all notions of multivalued mapping and define belief directly on the power set of the frame as in Smets' Transferable Belief Model \cite{smets93belief} as in (\ref{eq:belief-function}). 
To signal the fact that in his TBM no multivalued mapping is required Smets called mass functions \emph{basic belief assignments} (b.b.a.s).

In particular, the conditional b.f. $b_U(B|A)$ with b.b.a. 
\[
m_U(B|A) = \sum_{C \subseteq A^c} m(B \cup C), \quad B \subseteq A
\]
turns out to be the minimal commitment specialisation of $b$ such that the plausibility of the complementary event $A^c$ is nil \cite{Klawonn:1992:DBT:2074540.2074558}. 

In \cite{smets93jeffrey}, Smets pointed out the distinction between \emph{revision} and \emph{focussing} in the conditional process, and the way they lead to unnormalized and geometric \cite{suppes1977} conditioning, respectively. In these two scenarios he proposed some generalizations of Jeffrey's rule of conditioning \cite{jeffrey65book,shafer81jeffrey} to belief calculus.

\subsubsection{Other work}

The topic of conditioning in evidence theory has been studied by a number of other authors \cite{gong2018low}.

Klopotek and Wierzchon \cite{klopotek99lncs}, for instance, have provided a frequency-based interpretation for conditional belief functions. Tang and Zheng \cite{tang05dempster} have discussed the issue of conditioning in a multidimensional space. Lehrer \cite{lehrer05updating} has proposed a geometric approach to determining the conditional expectation of non-additive probabilities. Such a conditional expectation can then be applied for updating, and to introduce a notion of independence.

Among the most recent work in the area, Meester and Kerkvliet have recently proposed to look at the issue from a different perspective \cite{meester2019new,meester2019infinite}, by redeveloping and rederiving various notions of conditional belief functions, using a relative frequencies stance similar to that of \cite{klopotek99lncs}. The authors call the two main forms of conditioning \emph{contingent} and \emph{necessary} conditioning, respectively.
\\
Matuszewski and Klopotek \cite{matuszewski2017does}, on their hand, have investigated the empirical nature of Dempster's rule of combination, providing an original interpretation of conditional belief functions as belief functions given the manipulation of the original empirical data.
Novel Jeffrey-like conditioning rules have been proposed in \cite{han2017evidence}, whereas
the efficient computation of belief-theoretic conditionals, in particular Dempster and Fagin-Halpern conditionals, has been investigated by Polpitiya et al. in \cite{polpitiya2017efficient,polpitiya2019linear}.

Both Dezert et al. \cite{dezert2018total} and Zhou and Cuzzolin \cite{zhou2017total} have recently approached the problem of generalising the total probability theorem for belief functions.
The former have proposed their own fomulation of the total belief theorem, and used it to derive formal expressions for conditional belief functions, whereas the latter have shown a proof of existence for total belief functions under the assumption that Dempster conditioning is used.
From the point of view of manipulating conditionals, a message passing algorithm \cite{nguyen2017approximate} that approximates belief
updating on evidential networks with conditional belief functions has been proposed by Nguyen.

The various approaches to conditioning have been recently reviewed by Coletti et al. \cite{coletti2017bayesian}, who considered the problem of Bayesian inference under imprecise prior information in the form of a conditional belief function.

\subsubsection{Conditioning: A summary}

Table \ref{table:conditioning} summarises the behaviour of the main conditioning operators (including Smets's \emph{conjunctive} $b_{\text{\ocap}} (\cdot|B)$ and \emph{disjunctive} $b_{\text{\ocup}} (\cdot|B)$ conditioning) in terms of the degrees of belief and plausibility of the resulting conditional belief function.

\begin{table}[ht!]
\setlength\extrarowheight{10pt}
\setlength{\tabcolsep}{5pt}
\begin{center} 
\caption{Belief and plausibility values of the outcomes of various conditioning operators \label{table:conditioning}}
\begin{tabular}{|c|c|c|}
\hline
\textbf{Operator} & \textbf{Belief value} & \textbf{Plausibility value}
\\
\hline
Dempster's $\oplus$ & 
\small{$\displaystyle \frac{pl_b (B) - pl_b (B \setminus A)}{pl_b (B)}$ } & 
\small{$\displaystyle \frac{pl_b (A \cap B)}{pl_b (B)}$} 
\\
\hline
Credal $Cr$ & 
\small{$\displaystyle \frac{b (A \cap B)}{b(A \cap B) + pl_b (\bar{A} \cap B)}$} & 
\small{$\displaystyle \frac{pl_b (A \cap B)}{pl_b (A \cap B) + b(\bar{A} \cap B)}$}
\\
\hline
Geometric $G$ & 
\small{$\displaystyle \frac{b(A \cap B)}{b(B)}$ } & 
\small{$\displaystyle \frac{b(B) - b(B \setminus A)}{b(B)}$} 
\\
\hline
Conjunctive $\text{\ocap}$ & 
\small{$b(A \cup \bar{B})$, $A \cap B \neq \emptyset$} & 
\small{$pl_b (A \cap {B})$, $A \not \supset B$} 
\\
\hline
Disjunctive $\text{\ocup}$ & 
\small{$b(A)$, $A \supset B$} & 
\small{$pl_b (A)$, $A \cap B = \emptyset$} 
\\
\hline
\end{tabular}
\end{center}
\end{table}

Conditioning operators form a nested family, from the most committal to the least:
\[
\begin{array}{cl}
b_{\text{\ocup}}(\cdot|B) & \leq b_\text{Cr}(\cdot|B) \leq 
b_\oplus(\cdot|B) \leq b_{\text{\ocap}}(\cdot|B) 
\\
& \leq pl_{\text{\ocap}}(\cdot|B) \leq pl_\oplus(\cdot|B)
\leq pl_\text{Cr}(\cdot|B) \leq pl_{\text{\ocup}}(\cdot|B).
\end{array}
\]

\subsection{Geometry approach to uncertainty}  \label{sec:soa-geometric-approach}

The geometry of set functions and other uncertainty measures has been studied by several authors \cite{black97geometric,danilov00,maass06philosophical}. 
Indeed, geometry may potentially be a unifying language for the field of uncertainty theory \cite{cuzzolin18belief-maxent,Cuzzolin99, black97geometric, rota97book, ha98geometric, wang91geometrical}, possibly in conjunction with an algebraic view \cite{cuzzolin00rss,cuzzolin01bcc,cuzzolin08isaim-matroid,cuzzolin01lattice,cuzzolin05amai,cuzzolin07bcc,cuzzolin14algebraic}.
Recent papers on this topic include \cite{luo2020vector,pan2020probability,long2021visualization}.

In the {geometric approach to uncertainty}, uncertainty measures can be seen as points of a suitably complex geometric space, and there manipulated (e.g. combined, conditioned and so on) \cite{cuzzolin01thesis,cuzzolin2008geometric,cuzzolin2021springer}.
Much work has been focusing on the geometry of belief functions, which live in a convex space termed the \emph{belief space}, which can be described both in terms of a simplex (a higher-dimensional triangle) and in terms of a recursive bundle structure \cite{cuzzolin01space,cuzzolin03isipta,cuzzolin14annals,cuzzolin14lap}. The analysis can be extended to Dempster's rule of combination by introducing the notion of a conditional subspace and outlining a geometric construction for Dempster's sum \cite{cuzzolin02fsdk,cuzzolin04smcb}.
The combinatorial properties of plausibility and commonality functions, as equivalent representations of the evidence carried by a belief function, have also been studied \cite{cuzzolin08pricai-moebius,cuzzolin10ida}. The corresponding spaces are simplices which are congruent to the belief space.

Subsequent work has extended the geometric approach to other uncertainty measures, focusing in particular on possibility measures (consonant belief functions) \cite{cuzzolin10fss} and consistent belief functions \cite{cuzzolin11-consistent,cuzzolin09isipta-consistent,cuzzolin08isaim-simplicial}, in terms of simplicial complexes \cite{cuzzolin04ipmu}. Analyses of belief functions in terms credal sets have also been conducted \cite{cuzzolin08-credal,antonucci10-credal,burger10brest}.

The geometry of the relationship between measures of different kinds has also been extensively studied \cite{cuzzolin05hawaii,cuzzolin09-intersection,cuzzolin07ecsqaru,cuzzolin2010credal}, with particular attention to the problem of transforming a belief function into a classical probability measure \cite{Cobb03isf,voorbraak89efficient,Smets:1990:CPP:647232.719592}. One can distinguish between an `affine' family of probability transformations \cite{cuzzolin07smcb} (those which commute with affine combination in the belief space), and an `epistemic' family of transforms \cite{cuzzolin07report}, formed by the relative belief and relative plausibility of singletons \cite{cuzzolin08unclog-semantics,cuzzolin2008semantics,CUZZOLIN2012786,cuzzolin06-geometry,cuzzolin10amai}, which possess dual properties with respect to Dempster's sum \cite{cuzzolin2008dual}.
Semantics for the main probability transforms can be provided in terms of credal sets, i.e., convex sets of probabilities \cite{cuzzolin2010credal}. 
The problem of finding the possibility measure which best approximates a given belief function \cite{aregui08constructing} can also be approached in geometric terms \cite{cuzzolin09ecsqaru,cuzzolin11isipta-consonant,cuzzolin14lp,Cuzzolin2014tfs}. 

In particular, analogously to what done in this paper, approximations induced by classical Minkowski norms can be derived and compared with classical outer consonant approximations \cite{Dubois90}.
Minkowski consistent approximations of belief functions in both the mass and the belief space representations can also be derived \cite{cuzzolin11-consistent}.
\\
Preliminary studies on the application of the geometric approach to conditioning have been conducted by the author in 
\cite{cuzzolin10brest,cuzzolin11isipta-conditional}. 

\subsection{Distances in evidence theory}  \label{sec:soa-distances}

A number of norms for belief functions have been introduced 
\cite{diaz06fusion,jiang08new,khatibi10new,shi10distance}, as a tool for assessing the level of conflict between different bodies of evidence, for approximating a b.f. using a different uncertainty measure and so on.

Most relevantly to the proposed, Jousselme et al \cite{jousselme10brest} have conducted a very interesting survey of all the distances and similarity measures so far introduced in belief calculus, and proposed a number of generalisations. 
Generalisations to belief functions of the classical Kullback--Leibler divergence 
\[
D_{\text{KL}}(P|Q) = \int_{-\infty}^\infty p(x) \log({p(x)}/{q(x)}) \mathrm{d} x 
\]
of two probability distributions $P,Q$, for instance, have been proposed, together with measures based on information theory, such as fidelity, or entropy-based norms \cite{Florea2009metrics}.

The most popular and most cited measure of dissimilarity was proposed by Jousselme et al. \cite{Jousselme200191}. Jousselme's measure assumes that mass functions $m$ are represented as vectors $\vec{m}$, and reads as
\[
d_\text{J}(m_1,m_2) \doteq \sqrt{\frac{1}{2} (\vec{m}_1 - \vec{m}_2)^T D (\vec{m}_1 - \vec{m}_2)},
\]
where $D(A,B) = \frac{|A \cap B|}{|A \cup B|}$ for all $A,B \in 2^\Theta$.
Jousselme's distance so defined (1) is positive definite (as proved by Bouchard et al. in \cite{Bouchard2013615}), 
and thus defines a metric distance; (2) takes into account the similarity among subsets (focal elements); and (3) is such that $D(A,B)<D(A,C)$ if $C$ is `closer' to $A$ than $B$.

Other similarity measures between belief functions have been proposed by Shi et al \cite{shi10distance}, Jiang et al \cite{jiang08new}, and others \cite{khatibi10new,diaz06fusion,jiang08new}. Among others, it is worth mentioning the following proposals.
\begin{itemize}
\item
The Dempster conflict $\kappa$ and Ristic's closely related \emph{additive global dissimilarity measure} \cite{ristic06if}: $-\log (1 - \kappa)$.
\item
The `fidelity' or Bhattacharia coefficient \cite{10.2307/25047882} extended to belief functions, namely $\sqrt{\vec{m}_1}^T W \sqrt{\vec{m}_2}$, where $W$ is positive definite and 
$\sqrt{\vec{m}}$ is the vector obtained by taking the square roots of each component of $\vec{m}$.
\item
Perry and Stephanou's distance \cite{187371},
\[
d_\text{PS} (m_1,m_2) = |\mathcal{E}_1 \cup \mathcal{E}_2| \left ( 1 - \frac{\mathcal{E}_1 \cap \mathcal{E}_2}{\mathcal{E}_1 \cup \mathcal{E}_2} \right ) + (\vec{m}_{12} - \vec{m}_1)^T (\vec{m}_{12} - \vec{m}_2),
\]
where $\mathcal{E}$ is, as usual, the collection of focal elements of the belief function with mass assignment $m$ and $\vec{m}_{12}$ is the mass vector of the Dempster combination of $m_1$ and $m_2$.
\item
Blackman and Popoli's \emph{attribute distance} \cite{Blackman99},
\[
\begin{array}{l}
d_\text{BP} (m_1,m_2) 
\\
\displaystyle
= -2 \log \left [\frac{1 - \kappa(m_1,m_2)}{1 - \max_i \{ \kappa(m_i,m_i)\}} \right ] + (\vec{m}_1 + \vec{m}_2)^T \vec{g}_A - \vec{m}_1^T G \vec{m}_2,
\end{array}
\]
where $\vec{g}_A$ is the vector with elements $\vec{g}_A(A) = \frac{|A|-1}{|\Theta|-1}$, $A \subset \Theta$, and
\[
G(A,B) = \frac{(|A|-1)(|B|-1)}{(|\Theta|-1)^2}, \quad A,B \subset \Theta .
\]
\item
Fixen and Mahler's \emph{Bayesian percent attribute miss} \cite{fixen95modified}:
\[
\vec{m}_1' P \vec{m}_2,
\]
where $P(A,B) = \frac{p(A \cap B)}{p(A)p(B)}$ and $p$ is an a priori probability on $\Theta$.
\item
Zouhal and Denoeux's inner product of pignistic functions \cite{zouhal98evidence}.
\item
The family of information-based distances also proposed by Denoeux \cite{Denoeux01ijufk}, in which the distance between $m_1$ and $m_2$ is quantified by the difference between their information contents $U(m_1), U(m_2)$,
\[
d_U(m_1,m_2) = |U(m_1) - U(m_2)|,
\]
where $U$ is any uncertainty measure for belief functions.
\end{itemize}

Many if not all of these measures could be in principle plugged in the minimisation problem (\ref{eq:geometric-conditioning}) to define conditional belief functions. 
Nevertheless, in this paper we will restrict ourselves to the class of Minkowski measures.
Note that the $L_1$ distance was used earlier by Klir \cite{klir1999uncertainty} and Harmanec \cite{DBLP:journals/corr/abs-1301-6703}).


\section{Geometric conditional belief functions} \label{sec:conditional-m}

\subsection{The geometry of belief functions} \label{sec:geometric}


Given a frame $\Theta$, each belief function $b : 2^\Theta \rightarrow [0,1]$ is completely specified by its $N - 2$ belief values $\{ b(A), \emptyset \subsetneq A \subsetneq\Theta \}$, $N \doteq 2^{n}$ ($n \doteq |\Theta|$), (as $b(\emptyset) = 0$, $b(\Theta) = 1$ for all b.f.s) and can therefore be represented as a vector of $\mathbb{R}^{N-2}$ 
\[
\vec{b} = [b(A), \emptyset \subsetneq A \subsetneq\Theta]'. 
\]
If we denote by ${b}_A$ the \emph{categorical} \cite{smets94transferable} belief function assigning all the mass to a single subset $A\subseteq \Theta$, 
\[
m_{b_A}(A) = 1, 
\quad
m_{b_A}(B)=0 \; \forall B\subseteq \Theta, B\neq A, 
\]
we can prove that \cite{cuzzolin01space,cuzzolin2008geometric} the set of points of $\mathbb{R}^{N-2}$ which correspond to a b.f. or `belief space' $\mathcal{B}$ coincides with the convex closure $Cl$ of all the vectors representing categorical belief functions
\[
\mathcal{B} = Cl(\vec{b}_A, \emptyset \subsetneq A \subseteq \Theta),
\]
where
\[
Cl(\vec{b}_1,...,\vec{b}_k) = \left \{ \vec{b} \in \mathcal{B} : \vec{b} = \alpha_1 \vec{b}_1 + \cdots + \alpha_k \vec{b}_k, \sum_i \alpha_i = 1, \; \alpha_i\geq 0\; \forall i \right \}.
\]

The belief space $\mathcal{B}$ is a simplex \cite{cuzzolin01thesis,cuzzolin14lap,cuzzolin2021springer}, and each vector $\vec{b}\in\mathcal{B}$ representing a belief function $b$ can be written as a convex sum as:
\begin{equation} \label{eq:dev}
\vec{b} = \sum_{\emptyset \subsetneq B \subseteq \Theta} m_b(B) \vec{b}_B.
\end{equation}

In the same way, each belief function is uniquely associated with the related set of `mass' values $m_b(A)$. It can therefore be seen also as a point of $\mathbb{R}^{N-1}$, the vector $\vec{m}_b = [ m_b(A), \emptyset \subsetneq A \subseteq \Theta ]'$ ($\Theta$ this time included) of its $N-1$ mass components, which can be decomposed as
\begin{equation} \label{eq:dev-m}
\vec{m}_b = \sum_{\emptyset \subsetneq B \subseteq \Theta} m_b(B) \vec{m}_B,
\end{equation}
where $\vec{m}_B$ is the vector of mass values associated with the categorical belief function ${b}_B$. 

Note that in $\mathbb{R}^{N-1}$ $\vec{m}_\Theta = [0, ..., 0, 1]'$ cannot be neglected.

\subsection{Notion of geometric conditional belief functions} \label{sec:notion-conditional}

Similarly, the vector $\vec{m}_a$ associated with any belief function $a$ whose mass supports only focal elements $\{ \emptyset \subsetneq B \subseteq A \}$ included in a given event $A$ can be decomposed as:
\begin{equation} \label{eq:dev-a}
\vec{m}_a = \sum_{\emptyset \subsetneq B \subseteq A} m_a(B) \vec{m}_B.
\end{equation}
The set of such vectors is a simplex $\mathcal{M}_A \doteq Cl(\vec{m}_B, \; \emptyset \subsetneq B \subseteq A)$. We call $\mathcal{M}_A$ the \emph{conditioning simplex} in the mass space.

Given a belief function $b$, we call \emph{geometric conditional belief function induced by a distance function $d$} in $\mathcal{M}$ the b.f.(s) $b_{d,\mathcal{M}}(.|A)$ which minimize(s) the distance $d(\vec{m}_b,\mathcal{M}_A)$ between the mass vector representing $b$ and the conditioning simplex associated with $A$ in $\mathcal{M}$. 

As recalled above, a large number of proper distance functions or mere dissimilarity measures between belief functions have been proposed in the past, and many others can be imagined or designed \cite{jousselme10brest}.
We consider here as distance functions the three major $L_p$ norms $d=L_1$, $d=L_2$ and $d=L_\infty$. This is not to claim that these are \emph{the} distance functions of choice for this problem. 
In recent times, however, $L_p$ norms have been successfully employed in different problems such as probability \cite{cuzzolin07smcb} and possibility \cite{cuzzolin11-consistent,cuzzolin11isipta-consonant} transformation/approximation. 

For vectors $\vec{m}_b$, $\vec{m}_{b'} \in \mathcal{M}$ representing the b.p.a.s of two belief functions $b$, $b'$, such norms read as
\begin{equation} \label{eq:lp-m}
\begin{array}{l} 
\| \vec{m}_b - \vec{m}_{b'} \|_{L_1} \doteq \displaystyle \sum_{\emptyset \subsetneq B \subseteq \Theta} \Big | m_b(B) - m_{b'}(B) \Big |, 
\\
\| \vec{m}_b - \vec{m}_{b'} \|_{L_\infty} \doteq \displaystyle \max_{\emptyset \subsetneq B \subseteq \Theta} \Big |m_b(B) - m_{b'}(B) \Big | 
\\ 
\| \vec{m}_b - \vec{m}_{b'} \|_{L_2} \doteq \displaystyle \sqrt{\sum_{\emptyset \subsetneq B \subseteq \Theta} ( m_b(B) - m_{b'}(B))^2}.
\end{array}
\end{equation}

\subsection{Conditioning by $L_1$ norm} \label{sec:l1-m}

Given a belief function $b$ with basic probability assignment $m_b$ collected in a vector $\vec{m}_b \in \mathcal{M}$, its $L_1$ conditional version(s) $b_{L_1,\mathcal{M}}(.|A)$ has/have basic probability assignment $m_{L_1,\mathcal{M}}(.|A)$ s.t.:
\begin{equation} \label{eq:conditional-l1}
\vec{m}_{L_1,\mathcal{M}}(.|A) \doteq \arg \min_{\vec{m}_a \in \mathcal{M}_A} \| \vec{m}_b - \vec{m}_a \|_{L_1}.
\end{equation}
Using the expression (\ref{eq:lp-m}) of the $L_1$ norm in the mass space $\mathcal{M}$, (\ref{eq:conditional-l1}) becomes:
\[
\arg \min_{\vec{m}_a \in \mathcal{M}_A} \| \vec{m}_b - \vec{m}_a \|_{L_1} = \arg \min_{\vec{m}_a \in \mathcal{M}_A} \sum_{\emptyset \subsetneq B \subseteq \Theta} | m_b(B) - m_a(B)|.
\]
By exploiting the fact that the candidate solution $\vec{m}_a$ is an element of $\mathcal{M}_A$ (Equation (\ref{eq:dev-a})) we can greatly simplify this expression.

\begin{lemma} \label{lem:difference-vector-m}
The difference vector $\vec{m}_b - \vec{m}_a$ in $\mathcal{M}$ has the form:
\begin{equation} \label{eq:dev-beta2}
\begin{array}{l}
\displaystyle \vec{m}_b - \vec{m}_a = \sum_{\emptyset \subsetneq B \subsetneq A} \beta(B) \vec{m}_B + \Big ( b(A) - 1 - \sum_{\emptyset \subsetneq B \subsetneq A} \beta(B) \Big ) \vec{m}_A + \sum_{B \not\subset A} m_b(B) \vec{m}_B
\end{array}
\end{equation}
where $\beta(B) \doteq m_b(B) - m_a(B)$.
\end{lemma}
In the $L_1$ case therefore:
\begin{equation}\label{eq:l1-norm}
\| \vec{m}_b - \vec{m}_a \|_{L_1} = \sum_{\emptyset \subsetneq B \subsetneq A} |\beta(B)| + \Big | b(A) - 1 - \sum_{\emptyset \subsetneq B \subsetneq A} \beta(B) \Big |,
\end{equation}
plus the constant $\sum_{B \not\subset A} |m_b(B)|$. This is a function of the form
\begin{equation}\label{eq:form}
\sum_i |x_i| + \Big |-\sum_i x_i - k \Big |, \;\;\; k \geq 0
\end{equation}
which has an entire simplex of minima, namely: $x_i \leq 0$ $\forall i$, $\sum_i x_i \geq - k$. See Figure \ref{fig:l1} for the case of two variables, $x_1$ and $x_2$ (corresponding to the $L_1$ conditioning problem on an event $A$ of size $|A| = 2$).

\begin{figure}[htb]
\begin{center}
\begin{tabular}{cc}
\includegraphics[width = 0.6 \textwidth]{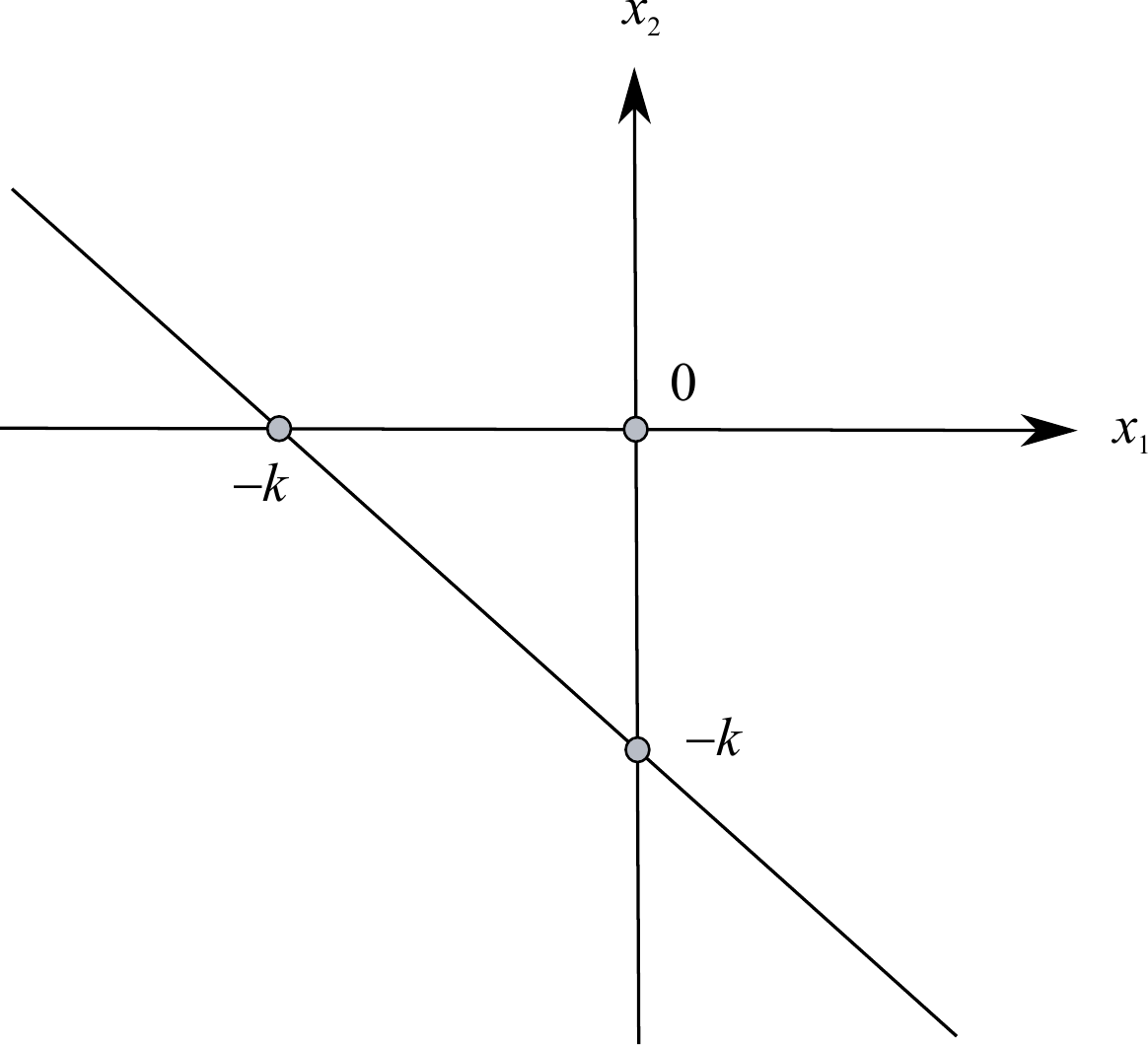} 
\end{tabular}
\end{center}
\caption{The minima of a function of the form (\ref{eq:form}) with two variables $x_1$, $x_2$ form the triangle $x_1\leq 0$, $x_2\leq 0$, $x_1 + x_2 \geq -k$ depicted here. 
\label{fig:l1} }
\end{figure}

A similar behaviour takes place in the general case too.

\begin{theorem} \label{the:l1-m}
Given a belief function $b:2^\Theta \rightarrow [0,1]$ and an arbitrary non-empty focal element $\emptyset \subsetneq A \subseteq \Theta$, the set of $L_1$ conditional belief functions $b_{L_1,\mathcal{M}}(.|A)$ with respect to $A$ in $\mathcal{M}$ is the set of b.f.s with core in $A$ such that their mass dominates that of $b$ over all the proper subsets of $A$:
\begin{equation}\label{eq:l1-conditional-solution}
b_{L_1,\mathcal{M}}(.|A) = \Big \{ a : 2^\Theta \rightarrow [0,1] : \mathcal{C}_{a} \subseteq A, \; m_{a}(B) \geq m_b(B) \; \forall \emptyset \subsetneq B \subseteq A \Big \}.
\end{equation}
\end{theorem}
As in the toy example of Figure \ref{fig:l1}, the set of $L_1$ conditional belief function in $\mathcal{M}$ has geometrically the form of a simplex.

\begin{theorem} \label{the:l1-simplex-m}
Given a b.f. $b:2^\Theta \rightarrow [0,1]$ and an arbitrary non-empty focal element $\emptyset \subsetneq A \subseteq \Theta$, the set of $L_1$ conditional belief functions $b_{L_1,\mathcal{M}}(.|A)$ with respect to $A$ in $\mathcal{M}$ is the simplex
\[
\mathcal{M}_{L_1,A}[b] = Cl(\vec{m}[b]|_{L_1}^B A, \emptyset \subsetneq B \subseteq A)
\]
whose vertex $\vec{m}[b]|_{L_1}^B A$, $\emptyset \subsetneq B \subseteq A$, has coordinates $\{ m_a(B) \}$ such that
\begin{equation}\label{eq:l1-vertices}
\left \{
\begin{array}{ll}
m_a(B) = m_b(B) + 1 - b(A) = m_b(B) + pl_b(A^c), & \\ m_a(X) = m_b(X) \hspace{18mm} \forall \emptyset \subsetneq X \subsetneq A, X\neq B. &
\end{array}
\right.
\end{equation}
\end{theorem}

It is important to notice that all the vertices of the $L_1$ conditional simplex fall inside $\mathcal{M}_A$ proper. In principle, some of them could have fallen in the linear space generated by $\mathcal{M}_A$ but outside the simplex $\mathcal{M}_A$, i.e., some of the solutions $m_a(B)$ could have been negative. This is indeed the case for geometrical b.f.s induced by other norms, as we will see in the following.

\subsection{Conditioning by $L_2$ norm} \label{sec:l2-m}

Let us now compute the analytical form of the $L_2$ conditional belief function(s) in the mass space. We make use of the form (\ref{eq:dev-beta2}) of the difference vector $\vec{m}_b - \vec{m}_a$, where again $\vec{m}_a$ is an arbitrary vector of the conditional simplex $\mathcal{M}_A$. In this case, though, it is convenient to recall that the minimal $L_2$ distance between a point $\vec{p}$ and a vector space is attained by the point $\hat{q}$ of the vector space $V$ s.t. the difference vector $\vec{p} -\hat{q}$ is orthogonal to all the generators $\vec{g}_i$ of $V$:
\[
\arg \min_{\vec{q} \in V} \| \vec{p} - \vec{q} \|_{L_2} = \hat{q} \in V : \langle \vec{p} -\hat{q}, \vec{g}_i \rangle = 0 \;\;\; \forall i
\]
whenever $\vec{p}\in\mathbb{R}^m$, $V = span(\vec{g}_i, i)$. 

This fact is used in the proof of Theorem \ref{the:l2-m}.
\begin{theorem} \label{the:l2-m}
Given a belief function $b:2^\Theta \rightarrow [0,1]$ and an arbitrary non-empty focal element $\emptyset \subsetneq A \subseteq \Theta$, the unique $L_2$ conditional belief function $b_{L_2,\mathcal{M}}(.|A)$ with respect to $A$ in $\mathcal{M}$ is the b.f. whose b.p.a. redistributes the mass $1-b(A)$ to each focal element $B\subseteq A$ in an equal way: $\forall \emptyset \subsetneq B \subseteq A$
\begin{equation} \label{eq:l2-solution}
\begin{array}{llll}
m_{L_2,\mathcal{M}}(B|A) & = & \displaystyle m_b(B) + \frac{1}{2^{|A|}-1} \sum_{B\not \subset A} m_b(B) = \displaystyle m_b(B) + \frac{pl_b(A^c)}{2^{|A|}-1}.
\end{array}
\end{equation}
\end{theorem}
According to Equation (\ref{eq:l2-solution}) the $L_2$ conditional belief function is unique, and corresponds to the mass function which \emph{redistributes the mass the original belief function assigns to focal elements not included in $A$ to each and all the subsets of $A$ in an equal, even way}. 

$L_2$ and $L_1$ conditional belief functions in $\mathcal{M}$ display a strong relationship.
\begin{theorem}
Given a belief function $b:2^\Theta \rightarrow [0,1]$ and an arbitrary non-empty focal element $\emptyset \subsetneq A \subseteq \Theta$, the $L_2$ conditional belief function $b_{L_2,\mathcal{M}}(.|A)$ with respect to $A$ in $\mathcal{M}$ is the center of mass of the simplex $\mathcal{M}_{L_1,A}[b]$ of $L_1$ conditional belief functions with respect to $A$ in $\mathcal{M}$.
\end{theorem}
\begin{proof}
By definition the center of mass of $\mathcal{M}_{L_1,A}[b]$, whose vertices are given by (\ref{eq:l1-vertices}), is the vector
\[
\frac{1}{2^{|A|}-1} \sum_{\emptyset \subsetneq B \subseteq A} \vec{m}[b]|_{L_1}^B A
\]
whose entry $B$ is given by $\displaystyle \frac{1}{2^{|A|}-1} \Big [m_b(B) (2^{|A|}-1) + (1 - b(A))\Big]$, i.e., (\ref{eq:l2-solution}).
\end{proof}

\subsection{Conditioning by $L_\infty$ norm} \label{sec:linf-m}

Similarly, we can use Equation (\ref{eq:dev-beta2}) to minimize the $L_\infty$ distance between the original mass vector $\vec{m}_b$ and the conditioning subspace $\mathcal{M}_A$. Let us recall it here for sake of readability:
\[
\begin{array}{l}
\displaystyle \vec{m}_b - \vec{m}_a = \sum_{\emptyset \subsetneq B \subsetneq A} \beta(B) \vec{m}_B + \sum_{B \not\subset A} m_b(B) \vec{m}_B + \bigg ( b(A) - 1 - \sum_{\emptyset \subsetneq B \subsetneq A} \beta(B) \bigg ) \vec{m}_A.
\end{array}
\]
Its $L_\infty$ norm reads as $\| \vec{m}_b - \vec{m}_a \|_{L_\infty} = $
\[
= \max \left \{ |\beta(B)|, \emptyset \subsetneq B\subsetneq A ; |m_b(B)|, B \not\subset A ; \bigg | b(A) - 1 - \sum_{\emptyset \subsetneq B \subsetneq A} \beta(B) \bigg | \right \}.
\]
As 
\[
\bigg | b(A) - 1 - \sum_{\emptyset \subsetneq B \subsetneq A} \beta(B) \bigg | = \bigg |\sum_{B \not \subset A} m_b(B) + \sum_{\emptyset \subsetneq B \subsetneq A} \beta(B) \bigg |
\]
the above norm simplifies as
\begin{equation} \label{eq:linfnorm}
\begin{array}{l}
\displaystyle \max \left \{ |\beta(B)|, \emptyset \subsetneq B\subsetneq A ; \;\; \max_{B \not\subset A} \{ m_b(B) \} ; \;\; \bigg | \sum_{B \not \subset A} m_b(B) + \sum_{\emptyset \subsetneq B \subsetneq A} \beta(B) \bigg | \right \}.
\end{array}
\end{equation}
This is a function of the form
\begin{equation} \label{eq:function}
\begin{array}{lll}
f(x_1,...,x_{m-1}) & = & \displaystyle \max \left \{ |x_i| \; \forall i, \bigg |\sum_i x_i + k_1 \bigg |, k_2 \right \},
\end{array}
\end{equation}
with $0 \leq k_2 \leq k_1 \leq 1$. 

Consider the case $m=3$. Such a function has two possible behaviors in terms of its minimal region in the plane $x_1,x_2$. 
\\
If $k_1\leq 3k_2$ its contour function has a
set of minimal points given by 
\[
x_i \geq -k_2, x_1 + x_2 \leq k_2 - k_1.
\]
In the opposite case $k_1 > 3k_2$ the contour function 
admits a single minimal point, located in $[-1/3 k_1, -1/3 k_1]$.

For an arbitrary number $m-1$ of variables $x_1,...,x_{m-1}$, the first case is such that $k_2 \geq k_1/m$, in which situation the set of minimal points of a function of the form (\ref{eq:function}) is such that 
\[
x_i \geq -k_2, \quad \sum_i x_i \leq k_2 - k_1,
\]
and forms a simplex with $m$ vertices. Each vertex $v^i$, $i \neq m$ has components 
\[
v^i(j) = -k_2 \; \forall j\neq i, 
\quad
v^i(i) = -k_1 + (m-1)k_2, 
\]
while obviously $v^m = [-k_2,\cdots,-k_2]'$. In the opposite case the unique minimal point is located in 
\[
[(-1/m) k_1,\cdots,(-1/m) k_1]'.
\]

This analysis applies to the norm (\ref{eq:linfnorm}) as follows.

\begin{theorem} \label{the:linf-m}
Given a belief function $b:2^\Theta \rightarrow [0,1]$ with b.p.a. $m_b$, and an arbitrary non-empty focal element $\emptyset \subsetneq A \subseteq \Theta$, the set of $L_\infty$ conditional belief functions $m_{L_\infty,\mathcal{M}}(.|A)$ with respect to $A$ in $\mathcal{M}$ forms the simplex
\[
\mathcal{M}_{L_\infty,A}[b] = Cl(\vec{m}[b]|_{L_\infty}^{\bar{B}}A, \; \bar{B} \subseteq A)
\]
with vertices
\begin{equation} \label{eq:vertices-linf-n-1}
\left \{ 
\begin{array}{l} 
\displaystyle 
\vec{m}[b]|_{L_\infty}^{\bar{B}} (B|A) = m_b(B) + \max_{C \not \subset A} m_b(C) 
\quad \quad 
\forall B \subseteq A, B \neq \bar{B} 
\\  
\displaystyle
\vec{m}[b]|_{L_\infty}^{\bar{B}}(\bar{B}|A) = m_b(\bar{B}) + \sum_{C \not \subset A} m_b(C) - (2^{|A|}-2) \max_{C \not \subset A} m_b(C)
\end{array} 
\right.
\end{equation}
whenever
\[
\max_{C \not \subset A} m_b(C) \geq \frac{1}{2^{|A|}-1} \sum_{C \not \subset A} m_b(C).
\]
It reduces to the single belief function
\[
m_{{L_\infty},\mathcal{M}}(B|A) = m_b(B) + \frac{1}{2^{|A|}-1} \sum_{C \not \subset A} m_b(C) \;\;\; \forall B \subseteq A
\]
whenever 
\[
\max_{C \not \subset A} m_b(C) < \frac{1}{2^{|A|}-1} \sum_{C \not \subset A} m_b(C). 
\]
The latter is the barycenter of the simplex of $L_\infty$ conditional b.f.s in the former case, and coincides with the $L_2$ conditional belief function (\ref{eq:l2-solution}).
\end{theorem}

Note that, as (\ref{eq:vertices-linf-n-1}) is not guaranteed to be non-negative, the simplex of $L_\infty$ conditional belief functions in $\mathcal{M}$ does not necessarily fall entirely inside the conditioning simplex $\mathcal{M}_A$, i.e., it may include \emph{pseudo belief functions} \cite{smets02ijar}. 

Looking at (\ref{eq:vertices-linf-n-1}) we can observe that vertices are obtained by assigning the maximum mass not in the conditioning event to all its subsets indifferently. Normalisation is then achieved, in opposition to what happens in Dempster's rule, by \emph{subtracting} the total mass in excess of 1 in the specific component $\bar{B}$. This behavior is exhibited by other geometric conditional b.f. as shown in the following.

\section{A case study: the ternary frame} \label{sec:example-m}

If $|A| = 2$, $A = \{ x,y \}$, the conditional simplex is 2-dimensional, with three vertices $\vec{m}_x$, $\vec{m}_{y}$ and $\vec{m}_{x,y}$. 

For a belief function $b$ on $\Theta = \{x,y,z\}$ Theorem \ref{the:l1-m} states that the vertices of the simplex $\mathcal{M}_{L_1,A}$ of $L_1$ conditional belief functions in $\mathcal{M}$ are:
\[
\begin{array}{llllll}
\vec{m}[b]|_{L_1}^{\{x\}} \{x,y\} & = & \big [m_b(x) + pl_b(z), & m_b(y), & m_b(x,y) & \big ]', 
\\ 
\vec{m}[b]|_{L_1}^{\{y\}} \{x,y\} & = & \big [m_b(x) , & m_b(y) + pl_b(z), & m_b(x,y) & \big ]', 
\\ 
\vec{m}[b]|_{L_1}^{\{x,y\}} \{x,y\} & = & \big [m_b(x) , & m_b(y) , & m_b(x,y) + pl_b(z) & \big ]'.
\end{array}
\]
Figure \ref{fig:example-m} shows such simplex in the case of a belief function $b$ on the ternary frame $\Theta = \{ x,y,z \}$ and basic probability assignment
\begin{equation}\label{eq:exbf}
\vec{m} = [0.2, \; 0.3, \; 0, \; 0, \; 0.5, \; 0]',
\end{equation}
i.e., $m_b(x) = 0.2$, $m_b(y) = 0.3$, $m_b(x,z) = 0.5$.

\begin{figure}[ht!]
\begin{center}
\includegraphics[width = 0.8 \textwidth]{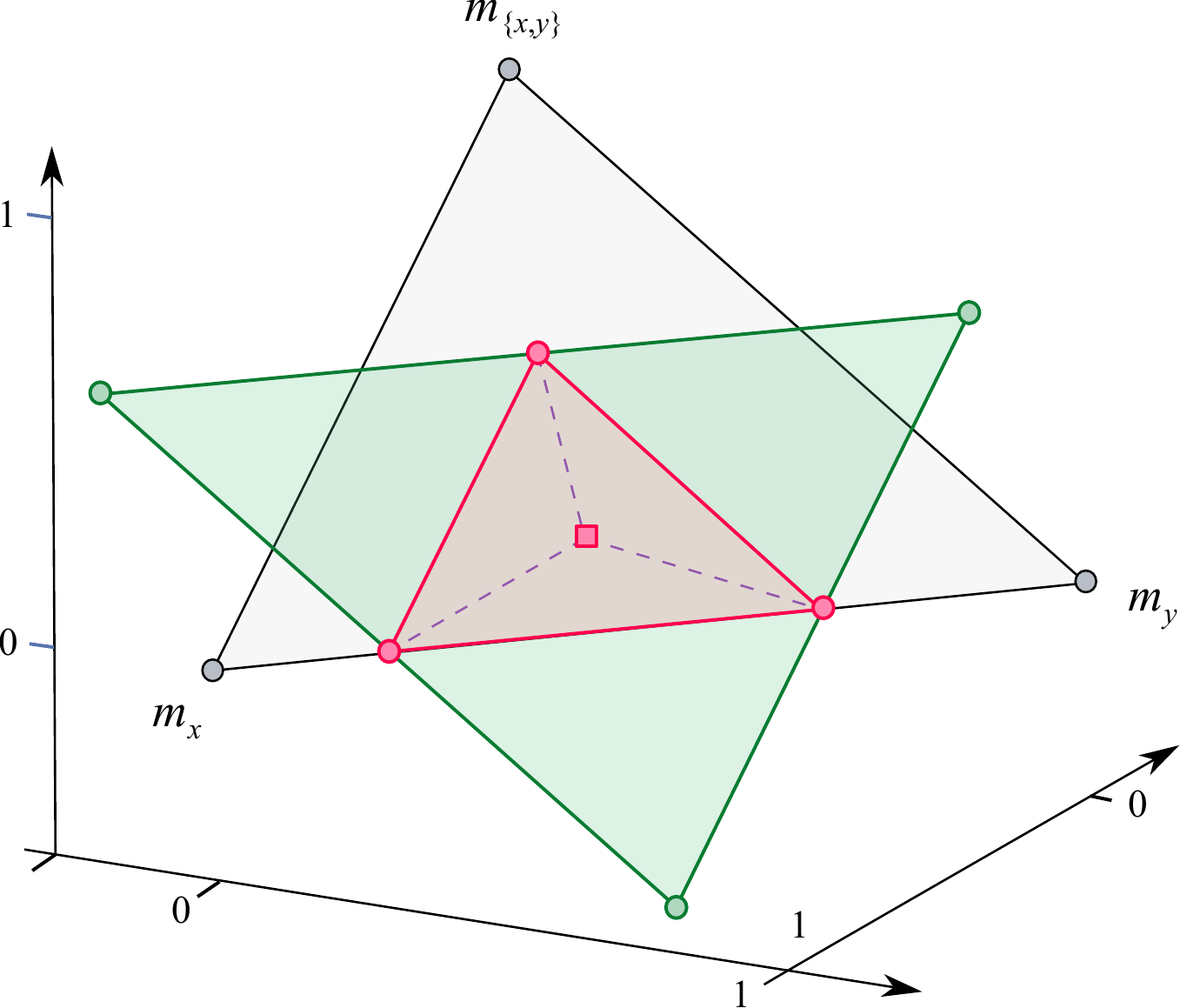}
\end{center}
\caption{The simplex (solid red triangle) of $L_1$ conditional belief functions in $\mathcal{M}$ associated with the belief function with mass assignment (\ref{eq:exbf}) in $\Theta = \{ x,y,z \}$. The related unique $L_2$ conditional belief function in $\mathcal{M}$ is plotted as a red square, and coincides with the center of mass of the $L_1$ set. The set of $L_\infty$ conditional (pseudo) belief functions is a simplex (triangle) with three vertices, depicted in green.}
\label{fig:example-m}
\end{figure}

In the case of the belief function (\ref{eq:exbf}) of the above example, by (\ref{eq:l2-solution}) its $L_2$ conditional belief function in $\mathcal{M}$ has b.p.a.
\begin{equation}\label{eq:l2-m-example}
\begin{array}{lll}
m (x) & = & \displaystyle m_b(x) + \frac{1 - b(x,y)}{3} = m_b(x) + \frac{pl_b(z)}{3}, 
\\
m (y) & = & \displaystyle m_b(y) + \frac{pl_b(z)}{3}, 
\\ 
m (x,y) & = & \displaystyle m_b(x,y) + \frac{pl_b(z)}{3}.
\end{array}
\end{equation}
Figure \ref{fig:example-m} visually confirms that such $L_2$ conditional belief function lies in the barycenter of the simplex of the related $L_1$ conditional b.f.s.

For what concerns $L_\infty$ conditional belief functions, the b.f. (\ref{eq:exbf}) is such that
\[
\begin{array}{lll}
\displaystyle \max_{C \not \subset A} m_b(C) & = & \displaystyle \max \Big \{ m_b(z), m_b(x,z), m_b(y,z), m_b(\Theta) \Big \} = m_b(x,z) \\ & = & \displaystyle 0.5 \geq \frac{1}{2^{|A|}-1} \sum_{C \not \subset A} m_b(C) = \frac{1}{3} m_b(x,z) = \frac{0.5}{3}.
\end{array}
\]
We hence fall within case 1, and there is a whole simplex of $L_\infty$ conditional belief function (in $\mathcal{M}$). According to Equation (\ref{eq:vertices-linf-n-1}) such simplex has $2^{|A|}-1 = 3$ vertices, namely (taking into account the nil masses in (\ref{eq:exbf}))
\begin{equation}\label{eq:linf-m-example}
\begin{array}{llllll}
\vec{m}[b]|_{L_\infty,\mathcal{M}}^{\{x\}} \{x,y\} & = & \big [m_b(x) - m_b(x,z) , & m_b(y) + m_b(x,z), & m_b(x,z) & \big ]', 
\\ 
\vec{m}[b]|_{L_\infty,\mathcal{M}}^{\{y\}} \{x,y\} & = & \big [ m_b(x) + m_b(x,z), & m_b(y) - m_b(x,z), & m_b(x,z) & \big]', 
\\ 
\vec{m}[b]|_{L_\infty,\mathcal{M}}^{\{x,y\}} \{x,y\} & = & \big [ m_b(x) + m_b(x,z), & m_b(y) + m_b(x,z), & - m_b(x,z) & \big ]'.
\end{array}
\end{equation}
We can notice that the set of $L_\infty$ conditional (pseudo) b.f.s is not entirely admissible, but its admissible part contains the set of $L_1$ conditional b.f.s, which amounts therefore a {more conservative} approach to conditioning. Indeed, the latter is the triangle inscribed in the former, determined by its median points.
Note also that both the $L_1$ and $L_\infty$ simplices have the same barycenter in the $L_2$ conditional belief function (\ref{eq:l2-m-example}).

\section{Discussion} \label{sec:discussion}

\subsection{Summary of results} \label{sec:discussion-summary}

To summarise, given a belief function $b:2^\Theta \rightarrow [0,1]$ and an arbitrary non-empty focal element $\emptyset \subsetneq A \subseteq \Theta$, we have the following.
\begin{enumerate}
\item
The set of $L_1$ conditional belief functions $b_{L_1,\mathcal{M}}(.|A)$ with respect to $A$ in $\mathcal{M}$ is the set of b.f.s with core in $A$ such that their mass dominates that of $b$ over all the subsets of $A$:
\[
b_{L_1,\mathcal{M}}(.|A) = \Big \{ a : \mathcal{C}_{a} \subseteq A, m_{a}(B) \geq m_b(B) \;\;\; \forall \emptyset \subsetneq B \subseteq A \Big \}.
\]
Such a set is a simplex $\mathcal{M}_{L_1,A}[b] = Cl(\vec{m}[b]|_{L_1}^B A, \emptyset \subsetneq B \subseteq A)$ whose vertices $\vec{m}_a = \vec{m}[b]|_{L_1}^B A$ have b.p.a.:
\[
\left \{
\begin{array}{l}
m_a(B) = m_b(B) + 1 - b(A) = m_b(B) + pl_b(A^c), \\ m_a(X) = m_b(X) \;\;\; \forall \emptyset \subsetneq X \subsetneq A, X\neq B .
\end{array}
\right.
\]
\item
The unique $L_2$ conditional belief function $b_{L_2,\mathcal{M}}(.|A)$ with respect to $A$ in $\mathcal{M}$ is the b.f. whose b.p.a. redistributes the mass $1-b(A) = pl_b(A^c)$ to each focal element $B\subseteq A$ in an equal way:
\begin{equation}\label{eq:l2-solution-m}
m_{L_2,\mathcal{M}}(B|A) = \displaystyle m_b(B) + \frac{pl_b(A^c)}{2^{|A|}-1},
\end{equation}
$\forall \emptyset \subsetneq B \subseteq A$, and corresponds to the center of mass of the simplex $\mathcal{M}_{L_1,A}[b]$ of $L_1$ conditional b.f.s.
\item
The $L_\infty$ conditional b.f. either coincides with the $L_2$ one, or forms a simplex obtained by assigning the maximal mass outside $A$ (rather than the sum of such masses $pl_b(A^c)$) to all subsets of $A$ (but one) indifferently.
\end{enumerate}

$L_1$ and $L_2$ conditioning are strictly related in the mass space, the latter being the barycenter of the former, and they have a compelling interpretation in terms of general imaging \cite{perea09amodel,Gardenfors}, as we argue next.

\subsection{Properties of geometric conditional belief functions} \label{sec:discussion-properties}

From our analysis a number of facts arise.
\begin{itemize}
\item
$L_p$ conditional belief functions, albeit obtained by minimising purely geometric distances, possess very simple and elegant interpretations in terms of degrees of belief.
\item
While some of them correspond to pointwise conditioning, some others form entire polytopes of solutions whose vertices also have simple interpretations.
\item
Conditional belief functions associated with the major $L_1$, $L_2$ and $L_\infty$ norms are strictly related to each other.
\item
In particular, while distinct, both the $L_1$ and $L_\infty$ simplices have barycenter in (or coincide with, in case 2) the $L_2$ conditional belief function.
\item
They are all characterized by the fact that, in the way they re-assign mass from focal elements $B \not \subset A$ not in $A$ to focal elements in $A$, they do not distinguish between subsets which have non-empty intersection with $A$ and those which have not.
\end{itemize}
The last point is quite interesting: mass-based geometric conditional b.f.s do not seem to care about the contribution focal elements make \emph{to the plausibility} of the conditioning event $A$, but only to whether they contribute or not to the \emph{degree of belief} of $A$. The reason is, roughly speaking, that in mass vectors $\vec{m}_b$ the mass of a given focal element appears only in the corresponding entry of $\vec{m}_b$. In opposition, belief vectors $\vec{b}$ are such that each entry 
\[
\vec{b}(B) = \sum_{X \subseteq B} m_b(X) 
\]
contains information about the mass of all the subsets of $B$. As a result, it could be expected that geometric conditioning \emph{in the belief space} $\mathcal{B}$ will see the mass redistribution process function in a manner linked to the contribution of each focal element to the plausibility of the conditioning event $A$. 

This is discussed in Section \ref{sec:discussion-belief}.

\subsection{Interpretation as general imaging} \label{sec:interpretation-imaging}

The form of geometric conditional belief functions in the mass space can be naturally interpreted in the framework of an interesting approach to belief revision, known as \emph{imaging} \cite{perea09amodel}. We will illustrate this notion and how it relates to our results using the example proposed in \cite{perea09amodel}. 

Suppose we briefly glimpse at a transparent urn filled with black or white balls, and are asked to assign a probability value to the possible `configurations' of the urn. Suppose also that we are given three options: 30 black balls and 30 white balls (state $a$); 30 black balls and 20 white balls (state $b$); 20 black balls and 20 white balls (state $c$). Hence, $\Theta = \{a,b,c\}$. Since the observation only gave us the vague impression of having seen approximately the same number of black and white balls, we would probably deem the states $a$ and $c$ equally likely, but at the same time we would tend to deem the event "$a$ or $c$" twice as likely as the state $b$. Hence, we assign probability 1/3 to each of the states. Now, we are told that state $c$ is false. How do we revise the probabilities of the two remaining states $a$ and $b$?

Lewis \cite{lewis76} argued that, upon observing that a certain state $x \in \Theta$ is impossible, we should transfer the probability originally allocated to $x$ to the remaining state deemed the `most similar' to $x$. In this case, $a$ is the state most similar to $c$, as they both consider an equal number of black and white balls. We obtain $(2/3,1/3)$ as probability values of $a$ and $b$, respectively. Peter G\"ardenfors further extended Lewis' idea (\emph{general imaging}) by allowing to transfer a part $\lambda$ of the probability 1/3, initially assigned to $c$, towards state $a$, and the remaining part $1-\lambda$ to state $b$. These fractions should be independent of the initial probabilistic state of belief.

Now, what happens when our state of belief is described by a belief function, and we are told that $A$ is true? In the general imaging framework we need to re-assign the mass $m(C)$ of each focal element not included in $A$ to all the focal elements $B \subseteq A$, according to some weights $\{ \lambda(B), B \subseteq A \}$. Suppose there is no reason to attribute larger weights to any focal element in $A$, as, for instance, we have no meaningful similarity measure (in the given context for the given problem) between the states described by two different focal elements. We can then proceed in two different ways. 
\\
One option is to represent our complete ignorance about the similarities between $C$ and each $B \subseteq A$ as a vacuous belief function on the set of weights. If applied to all the focal elements $C$ not included in $A$, this results in an entire polytope of revised belief functions, each associated with an arbitrary normalized weighting. It is not difficult to see that this coincides with the set $L_1$ conditional belief functions $b_{L_1,\mathcal{M}}(.|A)$ of Theorem \ref{the:l1-m}. On the other hand, we can represent the same ignorance as a uniform probability distribution on the set of weights $\{ \lambda(B), B \subseteq A \}$, for all $C \not\subset A$. Again, it is easy to see that general imaging produces in this case a single revised b.f., the $L_2$ conditional belief function $b_{L_2,\mathcal{M}}(.|A)$ of Theorem \ref{the:l2-m}.

As a final remark, the `information order independence' axiom of belief revision \cite{perea09amodel} states that the revised belief should not depend on the order in which the information is made available. In our case, the revised (conditional) b.f.s obtained by observing first an event $A$ and later another event $A'$ should be the same as the ones obtained by revising first with respect to $A'$ and then $A$. Both the $L_1$ and $L_2$ geometric conditioning operators presented here meet such axiom, supporting the case for their rationality.

\begin{figure}[ht!]
\begin{center}
\includegraphics[width = 0.8 \textwidth]{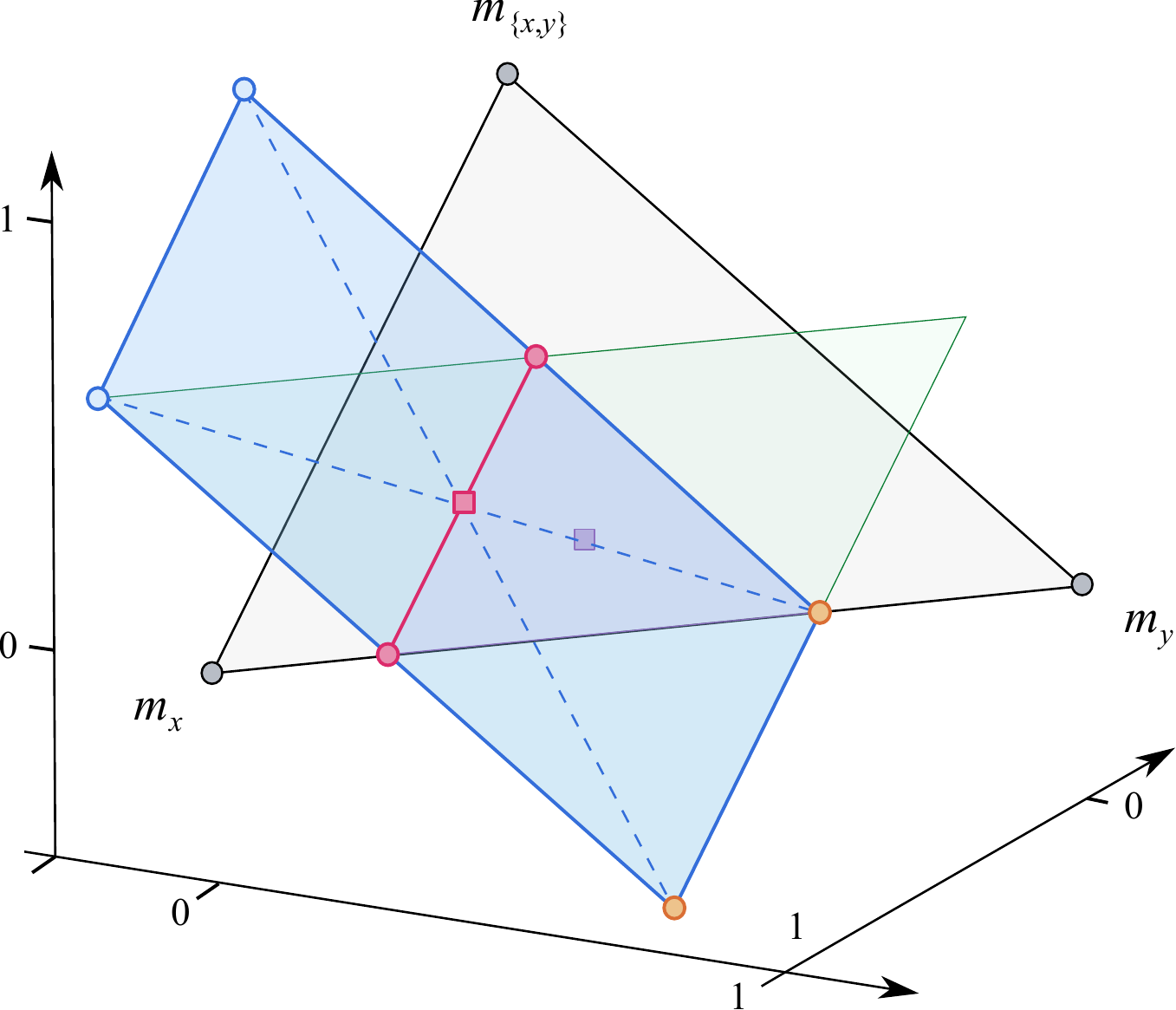}
\end{center}
\caption{
The set of $L_\infty,\mathcal{B}$ conditional belief functions is drawn here as a blue quadrangle for the belief function with mass assignment (\ref{eq:exbf}) in $\Theta = \{ x,y,z \}$, with conditioning event $A=\{x,y\}$. The region falls partly outside the conditioning simplex (gray triangle). The set of $L_1$ conditional belief functions in $\mathcal{B}$ is a line segment (in pink) with its barycentre in the $L_2$ conditional BF (pink square). In the ternary case, $L_2,\mathcal{B}$ is the barycentre of the $L_\infty,\mathcal{B}$ rectangle. 
The $L_p$ conditional belief functions in the mass space, already considered in Fig. \ref{fig:example-m}, are still visible in the background.
\label{fig:example-ternary-b}}
\end{figure}

\subsection{Comparison with conditioning in the belief space}  \label{sec:discussion-comparison}

\subsubsection{Results of conditioning in the belief space} \label{sec:discussion-belief}

Conditional belief functions can be derived by minimising Minkowski distances in the \emph{belief space} representation as well \cite{cuzzolin11isipta-conditional}. 
Unfortunately the results are of more difficult interpretation -- nevertheless a comparison with the more intuitive results obtained in the mass space is in place.

Namely \cite{cuzzolin11isipta-conditional,cuzzolin2021springer}, the $L_2$ conditional belief function and the barycentre $\overline{L_\infty}$ of the $L_\infty$ conditional belief functions computed in the belief space are, respectively,
\[
\begin{array}{lll}
\displaystyle 
m_{L_2,\mathcal{B}}(B|A) & = & \displaystyle m_b(B) + \sum_{C \subseteq A^c} m_b(B + C) 2^{-|C|} + (-1)^{|B|+1} \sum_{C \subseteq A^c} m_b (C) 2^{-|C|},
\\ \\
\displaystyle 
m_{\overline{L_\infty},\mathcal{B}}(B|A) & = & \displaystyle m_b (B) + \frac{1}{2} \sum_{\emptyset \subsetneq C \subseteq A^c} m_b (B + C) + \frac{1}{2} (-1)^{|B|+1} b (A^c).
\end{array}
\]

The $L_2$ result appears to be related to the process of mass redistribution (called by some authors \emph{specialisation}, \cite{Klawonn:1992:DBT:2074540.2074558,doi:10.1080/03081079008935126}) among all subsets, as happens with the ($L_2$-induced) orthogonal projection of a belief function onto the probability simplex \cite{cuzzolin07ecsqaru,cuzzolin07smcb}. In both expressions above, we can note that normalisation is achieved by alternately subtracting and summing a quantity, rather than via a ratio or, as in (\ref{eq:l2-solution-m}), by reassigning the mass of all $B \not \subset A$ to each $B \subsetneq A$ on an equal basis.

We can interpret the barycentre of the set of $L_\infty$ conditional belief functions, instead, as follows: the mass of all the subsets whose intersection with $A$ is $C \subsetneq A$ is re-assigned by the conditioning process {half to $C$}, and {half to $A$ itself}. In the case of $C=A$ itself, by normalisation, all the subsets $D \supseteq A$, including $A$, have their whole mass reassigned to $A$, consistently with the above interpretation. The mass $b(A^c)$ of the subsets which have no relation to the conditioning event $A$ is used to {guarantee the normalisation} of the resulting mass distribution. As a result, the mass function obtained is not necessarily non-negative: once again such a version of geometrical conditioning may generate pseudo-belief functions.

In the $L_1$ case obtaining a general analytic expression for the resulting conditional b.f.s appears impossible \cite{cuzzolin2021springer}. In the special cases in which this is possible, however, the result has potentially interesting interpretations. 


\subsubsection{Ternary example} \label{sec:discussion-belief-ternary}

Figure \ref{fig:example-ternary-b} illustrates the different geometric conditional belief functions as computed in the belief space, for a conditioning event $A = \{ x,y \}$ and a belief function with masses as in (\ref{eq:exbf}), i.e.
\[
m_b(x) = 0.2, \quad m_b (y) = 0.3, \quad m_b(x,z) = 0.5. 
\]

We already know that $m_{L_2,\mathcal{M}}(.|A)$ lies in the barycentre of the simplex of the $L_1, \mathcal{M}$ conditional belief functions. The same is true (at least in the ternary case) for $m_{L_2,\mathcal{B}}(.|A)$ (the pink square), which is the barycentre of the (blue) polytope of $m_{L_\infty,\mathcal{B}}(.|A)$ conditional belief functions. 
We can note that, as pointed out above, the latter does not fall entirely in the admissible conditional simplex $C(\vec{b}_x, \vec{b}_y, \vec{b}_{x,y})$ (although a significant portion does). Finding the admissible part of $m_{L_\infty, \mathcal{B}}(. | A)$ remains an open problem.

The set of $L_1$ conditional belief functions in $\mathcal{B}$ is, instead, a line segment (drawn in pink) whose barycentre is $m_{L_2,\mathcal{B}}(.|A)$. Such a set is:
\begin{itemize}
\item
entirely included in the set of $L_\infty$ approximations in both $\mathcal{B}$ and $\mathcal{M}$, thus representing a more conservative approach to conditioning;
\item 
entirely admissible.
\end{itemize}
It seems that, hard as it is to compute (see \cite{cuzzolin2021springer}, Chapter 15), $L_1$ conditioning in the belief space delivers interesting results. A number of interesting cross-relations between conditional belief functions in the two representation domains appear to exist:
\begin{enumerate}
\item
$m_{L_\infty,\mathcal{B}}(.|A)$ seems to contain $m_{L_1,\mathcal{M}}(.|A)$;
\item
the two $L_2$ conditional belief functions $m_{L_2,\mathcal{M}}(.|A)$ and $m_{L_2,\mathcal{B}}(.|A)$ appear to both lie on a line joining opposite vertices of $m_{L_\infty,\mathcal{B}}(.|A)$;
\item
$m_{L_\infty,\mathcal{B}}(.|A)$ (the blue polytope in Fig. \ref{fig:example-ternary-b}) and $m_{L_\infty,\mathcal{M}}(.|A)$ (the green triangle) have several vertices in common.
\end{enumerate}

There is probably more to these conditioning approaches than has been shown by the simple comparison done here. We will investigate these aspects further in the near future.

\section{Conclusions and perspectives} \label{sec:perspectives}

In this paper we showed how the notion of conditional belief function $b(.|A)$ can be introduced by geometric means, by projecting any belief function onto the simplex associated with the event $A$. The result will obviously depend on the choice of the vectorial representation for $b$, and of the distance function to minimize. We analyzed the case of conditioning a belief vector by means of the norms $L_1$, $L_2$ and $L_\infty$. This opens a number of interesting questions.

We may wonder, for instance, what classes of conditioning rules can be generated by such a distance minimization process. Do they span all known definitions of conditioning? In particular, is Dempster's conditioning itself a special case of geometric conditioning? We already mentioned Jousselme et al \cite{jousselme10brest} and their survey of the distance or similarity measures so far introduced between belief functions. Such a line of research could possibly be very useful in our quest.
A related question links geometric conditioning with combination rules \cite{yong2004combining}. Indeed, in the case of Dempster's rule it can be easily proven that \cite{cuzzolin04smcb},
\[
b \oplus b' = b \oplus \sum_{A \subseteq \Theta} m'(A) b_A = \sum_{A \subseteq \Theta} \mu(A) b \oplus b_A,
\]
where as usual $b'$ is decomposed as a convex combination of categorical belief functions $b_A$, and $\mu(A) \propto m'(A) pl_b(A)$. This means that Dempster's combination can be decomposed into a convex combination of Dempster's conditioning with respect to all possible events $A$. We can imagine to reverse this link, and generate combination rules from conditioning rules. Additional constraints have to be imposed in order to obtain a unique result. For instance, by imposing commutativity with affine combination (\emph{linearity}, in Smets' terminology \cite{smets95axiomatic}), any (geometrical) conditioning rule $b|^{\uplus}_A$ implies:
\[
b \uplus b' = \sum_{A \subseteq \Theta} m'(A) b \uplus b_A = \sum_{A \subseteq \Theta} m'(A) b|^{\uplus}_A.
\]

In the near future we plan to explore the world of combination rules induced by conditioning rules, starting from the different geometrical conditional processes introduced here.

\section*{Appendix}

\subsection*{Proof of Lemma \ref{lem:difference-vector-m}}

By definition 
\[
\vec{m}_b - \vec{m}_a = \sum_{\emptyset \subsetneq B \subseteq \Theta} m_b(B) \vec{m}_B - \sum_{\emptyset \subsetneq B \subseteq A} m_a(B) \vec{m}_B.
\]
The change of variables $\beta(B) \doteq m_b(B) - m_a(B)$ further yields:
\begin{equation} \label{eq:dev-beta}
\vec{m}_b - \vec{m}_a = \sum_{\emptyset \subsetneq B \subseteq A} \beta(B) \vec{m}_B + \sum_{B \not\subset A} m_b(B) \vec{m}_B.
\end{equation}
We observe, though, that the variables $\{ \beta(B), \emptyset \subsetneq B \subseteq A \}$ are not all independent. Indeed:
\[
\sum_{\emptyset \subsetneq B \subseteq A} \beta(B) = \sum_{\emptyset \subsetneq B \subseteq A} m_b(B) - \sum_{\emptyset \subsetneq B \subseteq A} m_a(B) = b(A) - 1
\]
as $\sum_{\emptyset \subsetneq B \subseteq A} m_a(B) = 1$ by definition, since $\vec{m}_a \in \mathcal{M}_A$. As a consequence, in the optimization problem (\ref{eq:conditional-l1}) there are just $2^{|A|}-2$ independent variables (as $\emptyset$ is not included), while 
\[
\beta(A) = b(A) - 1 - \sum_{\emptyset \subsetneq B \subsetneq A} \beta(B). 
\]
By replacing the above equality into (\ref{eq:dev-beta}) we get Equation (\ref{eq:dev-beta2}).

\subsection*{Proof of Theorem \ref{the:l1-m}}

The minima of the $L_1$ norm (\ref{eq:l1-norm}) are given by the set of constraints:
\begin{equation}\label{eq:constraints-l1}
\left \{ \begin{array}{ll} \beta(B) \leq 0 & \forall \emptyset \subsetneq B \subsetneq A \\ \displaystyle \sum_{\emptyset \subsetneq B \subsetneq A} \beta(B) \geq b(A) - 1. & \end{array} \right.
\end{equation}
In the original simplicial coordinates $\{ m_a(B), \emptyset \subsetneq B \subseteq A \}$ of the candidate solution $\vec{m}_a$ in $\mathcal{M}_A$ such system reads as:
\[
\left \{ \displaystyle m_b(B) - m_a(B) \leq 0 \;\; \forall \emptyset \subsetneq B \subsetneq A; \;\;\; \displaystyle \sum_{\emptyset \subsetneq B \subsetneq A} ( m_b(B) - m_a(B) ) \geq b(A) - 1, \right.
\]
i.e., $m_a(B) \geq m_b(B)$ $\forall \emptyset \subsetneq B \subseteq A$.

\subsection*{Proof of Theorem \ref{the:l1-simplex-m}}

By Equation (\ref{eq:constraints-l1}), the $2^{|A|}-2$ vertices of the simplex of $L_1$ conditional belief function in $\mathcal{M}$ (denoted by $\vec{m}[b]|_{L_1}^B A$, where $\emptyset \subsetneq B \subseteq A$) are determined by the following solutions:
\[
\begin{array}{l}
\vec{m}[b]|_{L_1}^A A : \left \{ \begin{array}{lll} \beta(X) & = & 0 \;\;\; \forall \emptyset \subsetneq X \subsetneq A, \end{array} \right.
\\
\vec{m}[b]|_{L_1}^B A : \left \{ \begin{array}{llll} \beta(B) & = & b(A) - 1, & \\ \beta(X) & = & 0 & \forall \emptyset \subsetneq X \subsetneq A, X\neq B. \end{array} \right. \forall \emptyset \subsetneq B \subsetneq A
\end{array}
\]
In coordinates $\{ m_a(B) \}$ the vertex $\vec{m}[b]|_{L_1}^B A$ is the vector $\vec{m}_a \in \mathcal{M}_A$ meeting Equation (\ref{eq:l1-vertices}).

\subsection*{Proof of Theorem \ref{the:l2-m}}

In the case that concerns us, $\vec{p} = \vec{m}_b$ is the original mass function, $\vec{q} = \vec{m}_a$ is an arbitrary point in $\mathcal{M}_A$, while the generators of $\mathcal{M}_A$ are all the vectors $\vec{g}_B = \vec{m}_B - \vec{m}_A$, $\forall \emptyset \subsetneq B \subsetneq A$. Such generators are vectors of the form
\[
[0,\cdots,0,1,0,\cdots,0,-1,0,\cdots,0]'
\]
with all zero entries but entry $B$ (equal to 1) and entry $A$ (equal to -1). Making use of Equation (\ref{eq:dev-beta}), the condition $\langle \vec{m}_b - \vec{m}_a, \vec{m}_B - \vec{m}_A \rangle = 0$ assumes then a very simple form
\[
\beta(B) - b(A) + 1 + \sum_{\emptyset \subsetneq X \subsetneq A, X \neq B} \beta(X) = 0
\]
for all possible generators of $\mathcal{M}_A$, i.e.:
\begin{equation} \label{eq:constraints-l2}
2 \beta(B) + \sum_{\emptyset \subsetneq X \subsetneq A, X \neq B} \beta(X) = b(A) - 1 \;\;\; \forall \emptyset \subsetneq B \subsetneq A.
\end{equation}
System (\ref{eq:constraints-l2}) is a linear system of $2^{|A|}-2$ equations in $2^{|A|}-2$ variables (the $\beta(X)$), that can be written as $\mathcal{A} \vec{\beta} = (b(A) - 1) \vec{1}$, where $\vec{1}$ is the vector of the appropriate size with all entries at 1. Its unique solution is trivially $\vec{\beta} = (b(A)-1) \cdot \mathcal{A}^{-1} \vec{1} $. The matrix $\mathcal{A}$ and its inverse are
\[
\begin{array}{cc}
\mathcal{A} = \left [ \begin{array}{cccc} 2 & 1 & \cdots & 1 \\ 1 & 2 & \cdots & 1 \\ & & \cdots & \\ 1 & 1 & \cdots & 2 \end{array} \right ].
&
\mathcal{A}^{-1} = \frac{1}{d+1} \left [ \begin{array}{cccc} d & -1 & \cdots & -1 \\ -1 & d & \cdots & -1 \\ & & \cdots & \\ -1 & -1 & \cdots & d \end{array} \right ],
\end{array}
\]
where $d$ is the number of rows (or columns) of $\mathcal{A}$. It is easy to see that $\mathcal{A}^{-1} \vec{1} = \frac{1}{d+1} \vec{1}$, where in our case $d = 2^{|A|}-2$. The solution to (\ref{eq:constraints-l2}) is then
\[
\vec{\beta} = \mathcal{A}^{-1} \vec{1} \cdot (b(A)-1) = \frac{1}{2^{|A|}-1} \vec{1} (b(A) - 1),
\]
or, more explicitly, 
\[
\beta(B) = \frac{b(A)-1}{2^{|A|}-1}
\quad
\forall
\emptyset \subsetneq B \subsetneq A.
\]
In the $\{ m_a(B) \}$ coordinates the $L_2$ conditional belief function reads as
\[
m_a(B) = m_b(B) + \frac{1 - b(A)}{2^{|A|}-1} = m_b(B) + \frac{pl_b(A^c)}{2^{|A|}-1} \hspace{5mm} \forall \emptyset \subsetneq B \subseteq A.
\]

\subsection*{Proof of Theorem \ref{the:linf-m}}

For the norm (\ref{eq:linfnorm}) the condition $k_2 \geq k_1/m$ for functions of the form (\ref{eq:function}) reads as:
\begin{equation}\label{eq:case-1}
\max_{C \not \subset A} m_b(C) \geq \frac{1}{2^{|A|}-1} \sum_{C \not \subset A} m_b(C).
\end{equation}
In such a case the set of $L_\infty$ conditional belief functions is given by the constraints $x_i \geq -k_2$, $\sum_i x_i \leq k_2 - k_1$, i.e.,
\[
\left \{ \displaystyle \beta(B) \geq - \max_{C \not \subset A} m_b(C) \;\; \forall B \subsetneq A, \;\;\; \displaystyle \sum_{B \subsetneq A} \beta(B) \leq \max_{C \not \subset A} m_b(C) - \sum_{C \not \subset A} m_b(C). \right.
\]
This is a simplex $Cl(\vec{m}[b]|^{L_\infty}_{\bar{B}} A, \bar{B} \subseteq A)$, where each vertex $\vec{m}[b]|^{L_\infty}_{\bar{B}} A$ is characterized by the following values $\vec{\beta}_{\bar{B}}$ of the auxiliary variables:
\[
\left \{ 
\begin{array}{l}
\displaystyle \vec{\beta}_{\bar{B}}(B) = - \max_{C \not \subset A} m_b(C) \;\; \forall B \subseteq A, B \neq \bar{B}; 
\\ 
\displaystyle \vec{\beta}_{\bar{B}}(\bar{B}) = - \sum_{C \not \subset A} m_b(C) + (2^{|A|}-2) \max_{C \not \subset A} m_b(C) 
\end{array}
\right.
\]
or, in terms of their basic probability assignments, (\ref{eq:vertices-linf-n-1}).

The barycenter of this simplex can be computed as follows:
\[
\begin{array}{lll} 
m_{\overline{L_\infty},\mathcal{M}}(B|A) 
& = & 
\displaystyle
\frac{1}{2^{|A|}-1} 
\sum_{\bar{B} \subseteq A} \vec{m}[b]|^{L_\infty}_{\bar{B}} (B|A)
\\ 
& = & \displaystyle
\frac{1}{2^{|A|}-1} 
\bigg [
\displaystyle (2^{|A|}-1) m_{b}(B) + (2^{|A|}-2)\max_{C \not \subset A} m_b(C) 
\\
& & \displaystyle
+ \sum_{C \not \subset A} m_b(C) - (2^{|A|}-2) \max_{C \not \subset A} m_b(C)
\bigg ]
\\ 
& = &
\displaystyle 
\frac{\displaystyle (2^{|A|}-1) m_{b}(B) + \sum_{C \not \subset A} m_b(C)}{2^{|A|}-1} = m_b(B) +  \frac{\displaystyle \sum_{C \not \subset A} m_b(C)}{2^{|A|}-1},
\end{array}
\]
i.e., the $L_2$ conditional belief function (\ref{eq:l2-solution}). The corresponding minimal $L_\infty$ norm of the difference vector is, according to (\ref{eq:linfnorm}), equal to $\max_{C \not \subset A} m_b(C)$.\\
The opposite case reads as
\begin{equation}\label{eq:case-2}
\max_{C \not \subset A} m_b(C) < \frac{1}{2^{|A|}-1} \sum_{C \not \subset A} m_b(C).
\end{equation}
For system (\ref{eq:linfnorm}) the unique solution is 
\[
\beta(B) = -\frac{1}{2^{|A|}-1} \sum_{C \not \subset A} m_b(C)
\quad
\forall B \subsetneq A 
\]
 or, in terms of basic probability assignments,
\[
m_{{L_\infty},\mathcal{M}}(B|A) = m_b(B) + \frac{1}{2^{|A|}-1} \sum_{C \not \subset A} m_b(C) \;\;\; \forall B \subseteq A.
\]
The corresponding minimal $L_\infty$ norm of the difference vector is in this second case equal to 
\[
\frac{1}{2^{|A|}-1} \sum_{C \not \subset A} m_b(C).
\]

\footnotesize
\bibliographystyle{plain}
\bibliography{arxiv-geometric-conditioning}

\end{document}